\documentclass[conference,compsoc]{IEEEtran}
\ifCLASSOPTIONcompsoc
  \usepackage[nocompress]{cite}
\else
  \usepackage{cite}
\fi
\ifCLASSINFOpdf
\else
\fi

\usepackage{dsfont}
\usepackage{stmaryrd}
\usepackage{bm}
\usepackage{amsmath}
\usepackage{url}
\usepackage{algorithm}
\usepackage{algorithmic}
\usepackage{amssymb}
\usepackage{amsthm}
\usepackage{xcolor}
\usepackage{verbatim}
\usepackage{multirow}
\usepackage{mathtools}
\usepackage{enumerate}
\usepackage{bibentry}
\usepackage{multirow}
\usepackage{setspace}
\usepackage{arydshln}
\usepackage{hyperref}
\usepackage{slashbox}

\theoremstyle{plain}
\newtheorem{theorem}{Theorem}
\newtheorem{lemma}[theorem]{Lemma}

\begin{document}
%
\title{Learning with Inadequate and Incorrect Supervision}

\author{\IEEEauthorblockN{Chen Gong\IEEEauthorrefmark{1}, Hengmin Zhang\IEEEauthorrefmark{1}, Jian Yang\IEEEauthorrefmark{1} and Dacheng Tao\IEEEauthorrefmark{2}}
\IEEEauthorblockA{\IEEEauthorrefmark{1}School of Computer Science and Engineering, Nanjing University of Science and Technology, China}
\IEEEauthorblockA{\IEEEauthorrefmark{2}UBTECH Sydney AI Centre, SIT, FEIT, University of Sydney, Australia}
\IEEEauthorblockA{Corresponding author: Chen Gong (E-mail: chen.gong@njust.edu.cn)}}

\maketitle

\begin{abstract}
Practically, we are often in the dilemma that the labeled data at hand are inadequate to train a reliable classifier, and more seriously, some of these labeled data may be mistakenly labeled due to the various human factors. Therefore, this paper proposes a novel semi-supervised learning paradigm that can handle both label insufficiency and label inaccuracy. To address label insufficiency, we use a graph to bridge the data points so that the label information can be propagated from the scarce labeled examples to unlabeled examples along the graph edges. To address label inaccuracy, Graph Trend Filtering (GTF) and Smooth Eigenbase Pursuit (SEP) are adopted to filter out the initial noisy labels. GTF penalizes the $\ell_0$ norm of label difference between connected examples in the graph and exhibits better local adaptivity than the traditional $\ell_2$ norm-based Laplacian smoother. SEP reconstructs the correct labels by emphasizing the leading eigenvectors of Laplacian matrix associated with small eigenvalues, as these eigenvectors reflect real label smoothness and carry rich class separation cues. We term our algorithm as ``\underline{S}emi-supervised learning under \underline{I}nadequate and \underline{I}ncorrect \underline{S}upervision'' (SIIS). Thorough experimental results on image classification, text categorization, and speech recognition demonstrate that our SIIS is effective in label error correction, leading to superior performance to the state-of-the-art methods in the presence of label noise and label scarcity.
\vskip 10pt
\end{abstract}


\IEEEpeerreviewmaketitle

\section{Introduction}
\label{sec:intro}

Practically, it is quite often that the available labeled data are insufficient for training a reliable supervised classifier such as Support Vector Machines (SVM) and Convolutional Neural Networks (CNN). For example, manually annotating web-scale images/texts is intractable because of the unacceptable human labor cost. Acquiring sufficient labeled examples for protein structure categorization is also infeasible as it often takes months of laboratory work for experts to identify a single protein's 3D structure. To make the matter worse, a portion of such limited labeled data are very likely to be mislabeled, which means that the sparse supervision information we have may not be reliable. For example, in crowdsourced image annotation, some of the image labels can be incorrect due to the knowledge or cultural limitation of the annotators. The labeling of protein structure is also error-prone as this process is highly depended on the experience and expertise of labelers working in the biological area.\par

To solve the abovementioned practical problems, this paper studies how to leverage the scarce labeled examples with untrustable labels to build a reliable classifier so that the massive unlabeled examples can be accurately classified. Therefore, two issues are jointly taken into consideration in this paper: one is the insufficiency of labeled examples, and the other is the noise in label space.\par

In fact, Semi-Supervised Learning (SSL) \cite{Goldberg09introductionto} has been widely used to deal with the first issue. SSL aims to predict the labels of a large amount of unlabeled examples given only a few labeled examples, and the algorithms of SSL can be roughly divided into three categories, i.e. \emph{collaboration-based}, \emph{large-margin-based}, and \emph{graph-based}. Collaboration-based methods usually contain multiple learners and they are trained collaboratively to improve the integrated performance on the unlabeled data. Co-training \cite{blum1998combining} and Tri-training \cite{zhou2005tritraining} are representative methodologies belonging to this category. Large-margin-based methods assume that there exists an optimal decision boundary in the low density region between data clusters, so that the margin between the decision boundary and the nearest data points on each side can be maximized. The algorithms based on the large-margin assumption are usually the variants of traditional SVM, such as Semi-Supervised SVM (S3VM) \cite{joachims1999transductive}, Mean S3VM \cite{li2009semi}, and Safe S3VM (S4VM) \cite{li2015towards}. Graph-based methods are usually established on the manifold assumption, namely the entire dataset contains a potential manifold, and the labels of examples should vary smoothly along this manifold. The representative algorithms include Harmonic Functions \cite{Zhu03semisupervisedlearning}, Local and Global Consistency \cite{Zhou03learningwith}, Linear Neighborhood Propagation \cite{wang2008labelTKDE}, and Manifold Regularization \cite{belkin2006manifold}. Although above SSL methods have achieved satisfactory results for different purposes, none of them are applicable to the label noise situations, and their performance will significantly decrease in the presence of mislabeled examples.\par

\begin{figure}[t]
  \centering
  \includegraphics[width=\linewidth]{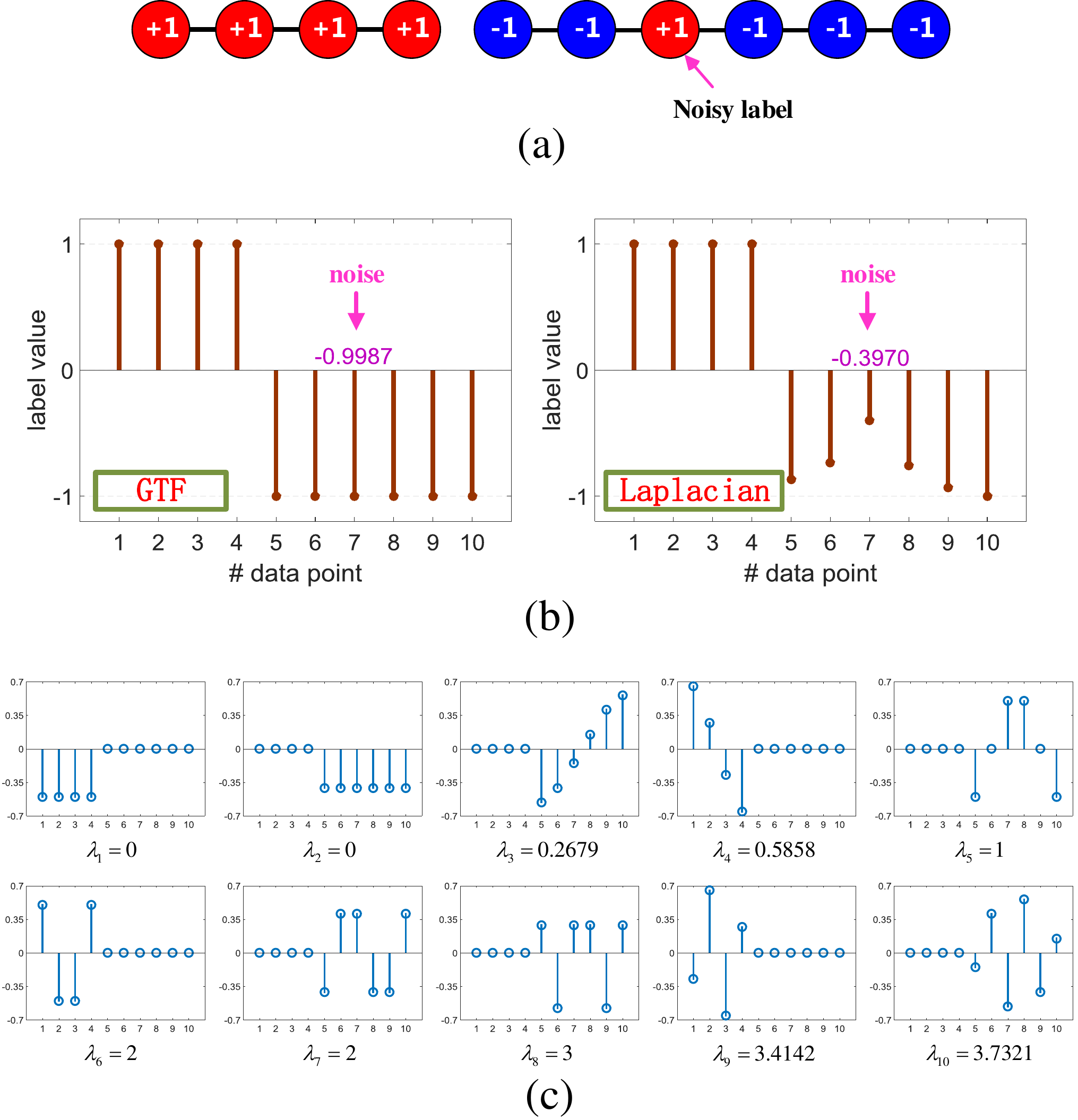}\\
  \caption{Illustration of our motivation. (a) presents an unweighted chain graph with 10 examples, and one of the negative examples is incorrectly labeled as positive. (b) shows the normalized labels of the 10 examples decided by GTF and Laplacian smoother, respectively, in which the noisy label is correspondingly corrected to -0.9987 and -0.3970. (c) plots the spectrum (i.e. eigenvalues and eigenvectors) of the Laplacian matrix associated with the graph in (a).}\label{fig:IdeaIllustration}
\end{figure}

Regarding the issue of label noise, several works have been done recently to prevent the performance degradation caused by such incorrect supervision. By discovering that most of the existing loss functions can be decomposed as a label-independent term plus a label-dependent term, Gao et al. \cite{gao2016risk} and Patrini et al. \cite{patrini2016loss} estimate the unbiased geometric mean of the entire dataset to suppress the negative influence of noisy labels. More generally, a variety of methods have been proposed to adapt the existing loss functions to the corrupted labels via weighting \cite{liu2016classification,natarajan2013learning}, calibration \cite{van2015learning}, or upper bounding \cite{han2016convergence}. Other works towards mitigating the label noise are based on graphical model \cite{xiao2015learning}, boosting \cite{miao2016rboost}, or data cleansing \cite{zhu2006bridging}. However, they are designed for supervised classifier and thus are not suitable for the SSL problem considered in this paper.\par

Therefore, in this paper we aim to design a semi-supervised algorithm that is robust to label noise, so that the unlabeled examples can be accurately classified although the labels at hand might be scarce and inaccurate. Consequently, our method is termed ``\textbf{S}emi-supervised learning under \textbf{I}nadequate and \textbf{I}ncorrect \textbf{S}upervision'' (SIIS). Mathematically, suppose we have $l$ labeled examples $\mathcal{L}\!=\!\left\{ \left( {{\mathbf{x}}_{i}},{{y}_{i}} \right) \right\}_{i=1}^{l}$ and $u$ unlabeled examples $\mathcal{U}\!=\!\left\{ {{\mathbf{x}}_{i}} \right\}_{i=l+1}^{n}$ with $n=l+u$. The labels ${{y}_{i}}$ for $1\!\le\! i\!\le \!l$ take values from $\left\{ -1,1 \right\}$ and some of them may be incorrect, then our goal is to classify the examples in $\mathcal{U}$ based on the inaccurate $\mathcal{L}$. Specifically, we build a graph $\mathcal{G}=\langle \mathcal{V}, \mathcal{E}\rangle$ where $\mathcal{V}$ is the vertex set consisted of all $n$ examples, and $\mathcal{E}$ is the edge set encoding the similarity between these examples. Based on $\mathcal{G}$, we adopt two measures to deal with the possible label noise, namely Graph Trend Filtering (GTF) and Smooth Eigenbase Pursuit (SEP).\par

GTF \cite{wang2016trend} is a statistical method to conduct the nonparametric regression on a graph. Its main idea is to penalize the $\ell_0$ norm of label difference between graph vertices rather than using the usual $\ell_2$ norm-based graph Laplacian smoother \cite{Zhou03learningwith,gong2014fick,gong2015deformed}. Consequently, the label difference between the connected vertices can be exactly zero by employing GTF. In contrast, the $\ell_2$ norm-based Laplacian smoother only decides the vertex difference to be small or large and can hardly set any label difference to exact zero. Therefore, GTF has a stronger power on correcting the noisy labels and achieves better local adaptivity than the traditional Laplacian smoother. In Fig.~\ref{fig:IdeaIllustration}(a), a chain graph with 10 vertices is presented, in which a negative vertex has been mislabeled as positive. Fig.~\ref{fig:IdeaIllustration}(b) shows the labels of the 10 vertices assigned by GTF and Laplacian smoother, respectively. It can be clearly observed that the negative data points, including the original mislabeled vertex, obtain very confident labels that are almost -1 by GTF. In contrast, Laplacian smoother is strongly affected by the label noise and thus the label assignments to the negative examples are significantly deviated from the ideal value -1. Specifically, the initial noisy label +1 is corrected to -0.3970 and -0.9987 by Laplacian smoother and GTF, respectively, which indicates that GTF is more powerful than Laplacian smoother on noisy label correction.\par

SEP stems from the spectral graph theory, which claims that the eigenvectors of graph Laplacian matrix corresponding to the smallest eigenvalues reflect the real underlying smoothness of labels and usually contain clear indication of class separations. For example, Fig.~\ref{fig:IdeaIllustration}(c) plots all 10 eigenvectors of the Laplacian matrix associated with the graph in (a). In Fig.~\ref{fig:IdeaIllustration}(c), the two leading eigenvectors are piecewise constant, which exactly correspond to the correct class separation of the 10 vertices regardless of the initial noisy label. However, when the eigenvalues increase, the corresponding eigenvectors become more and more rugged and the class information turns to be rather unclear. Therefore, the noisy labels can be easily filtered out if the smooth eigenbasis of a graph can be successfully discovered.\par

Thanks to the collaboration of GTF and SEP, our SIIS model performs robustly to the label noise. The experimental results on the datasets from different domains indicate that SIIS can not only effectively correct the corrupted labels on labeled set $\mathcal{L}$, but also achieve higher classification accuracy on unlabeled set $\mathcal{U}$ than the state-of-the-art methods.\par

This is the longer version of our previous conference submission \cite{gong2017ICDMlearning}. Compared to \cite{gong2017ICDMlearning}, this version contains more illustrative presentations, model descriptions, theoretical analyses, and empirical studies.

\section{Related Work}
\label{sec:RelatedWork}
As this paper cares about semi-supervised learning and the issue of label noise, this section will briefly review some representative works on these two topics.

\subsection{Semi-supervised Learning}
As mentioned in the introduction, SSL is specifically proposed for the situation where the labeled examples are scarce while the unlabeled examples are abundant. In SSL, although the unlabeled examples do not have explicit labels, they carry the distribution information of the entire dataset, so they also render the important cues for achieving accurate classification. SSL methods can be collaboration-based, large-margin-based, or graph-based.\par

Collaboration-based methods involve multiple classifiers so that they complement to each other and yield the final satisfactory performance. Blum et al. \cite{blum1998combining} firstly propose a basic co-training framework in which two classifiers exchange their individual most confident predictions in each iteration so that the unlabeled examples can be precisely classified with their interactions. 
Considering that the involved two classifiers may generate contradictory results under some uncertain circumstances, tri-training \cite{zhou2005tritraining} that incorporates three classifiers was proposed to eliminate such ambiguity.\par

Large-margin-based methods suppose that the classes in the example space are well-separated into several clusters \cite{shao2016reliable}, and the decision boundary should fall into a low density region between the clusters. For example, Semi-Supervised SVM (S3VM) \cite{joachims1999transductive} replaces the \emph{hinge loss} adopted by traditional SVM with a novel \emph{hat loss} to penalize the classification error on the unlabeled examples. To guarantee that the unlabeled examples will not hurt the performance, Safe Semi-Supervised SVM (S4VM) \cite{li2015towards} was proposed to simultaneously exploit multiple candidate low-density separators to reduce the risk of identifying a poor separator. Besides, considering that the SVM formulation cannot output the label posterior probability when making classification, entropy regularization \cite{grandvalet2004semi} was proposed by extending the conventional logistic regression to semi-supervised cases. Other works belonging to this type include 
\cite{li2009semi,chapelle2005semi,collobert2006large,ogawa2013infinitesimal}, etc.\par

Graph-based methods usually build a graph to approximate the manifold embedded in the dataset, and require the labels of examples vary smoothly along the manifold. Zhu et al. \cite{Zhu03semisupervisedlearning} and Zhou et al. \cite{Zhou03learningwith} exploited the un-normalized Graph Laplacian and normalized Graph Laplacian, respectively, to describe the smoothness of labels on the graph. Wang et al. \cite{wang2008labelTKDE} assume that a data point can be linearly reconstructed by its neighbors and build a LLE-like graph \cite{roweis2000nonlinear} to model the smoothness. Gong et al. \cite{gong2014fick} formulate the graph-based SSL as the fluid diffusion, and adopt a physical theory to achieve label smoothness. Different from above methods that simply focus on instance-level smoothness, Zhao et al. \cite{zhao2015part} adopt non-negative matrix factorization to discover the part-level structure in the graph. Other typical works include \cite{belkin2006manifold,gong2015deformed,wang2013dynamic,jia2016adaptive,yamaguchi2016camlp}, etc.

\subsection{Label Noise Handling}
The early-stage approaches for addressing the label noise issue usually focus on noise detection and filtering, namely the data are preprocessed to remove the possible noise ahead of conducting the standard algorithms. To this end, some algorithms exploit the neighborhood information \cite{muhlenbach2004identifying,fine1999noise}, while some algorithms utilize the prediction disagreements among an ensemble of classifiers \cite{brodley1999identifying,zhu2003eliminating}.\par

Recently, more efforts have been put to develop the algorithms that are inherently robust to label noise. Under the framework of risk minimization, a series of works \cite{natarajan2013learning,ghosh2015making,van2015learning,liu2016classification,gao2016risk,jindal2016learning} have been done to design various surrogate loss functions, so that the surrogate loss for noisy data is the same as the risk under the original loss for noise-free data. Apart from model designing, there are also some works \cite{patrini2016loss,han2016convergence} targeting to improve the robustness of model optimization. They show that the vanilla stochastic gradient descent method can be made stable to label noise if some minor modifications are made. More detailed survey on the topic of label noise can be found in \cite{frenay2014classification}.\par

So far some preliminary investigations have been done to solve the problem of SSL with deteriorated labels. By imposing the eigen-decomposition on the graph Laplacian matrix, Lu et al. \cite{lu2015noise} designed a new $\ell_1$ norm smoothness and transformed their model to a sparse coding problem. Gu et al. \cite{gu2016robust} utilized self-paced learning to progressively select the labeled examples in a well-organized manner, so that the initial noisy labels can be filtered out. Wang et al. \cite{wang2009labeldiagnosis} eliminated contradictory labels and conducted label inference through a bidirectional and alternating optimization strategy. Yan et al. \cite{yan2016robust} adopted multiple kernel learning and combined the outputs of weak SSL classifiers to approximate the ground-truth labels. However, these methods are not sufficiently robust, namely they cannot generate consistent satisfactory performance in the presence of different levels of label noise.

\section{Our Model}
\label{sec:OurModel}
In our method, we construct a $K$-nearest neighborhood ($K$NN) graph $\mathcal{G}$ over $\mathcal{L}\cup\mathcal{U}$, which is further quantified by the adjacency matrix $\mathbf{W}$. The $(i,j)$-th element of $\mathbf{W}$, i.e. $W_{ij}$, encodes the similarity between the examples $\mathbf{x}_i$ and $\mathbf{x}_j$, which is computed by ${{W}_{ij}}\!=\!\exp \left( -{\left\| {{\mathbf{x}}_{i}}\!-\!{{\mathbf{x}}_{j}} \right\|^{2}}/{(2{{\xi }^{2}})} \right)$ with $\xi $ being the Gaussian kernel width if ${{\mathbf{x}}_{i}}$ and ${{\mathbf{x}}_{j}}$ are linked by an edge in $\mathcal{G}$, and ${{W}_{ij}}\!=\!0$ otherwise. Based upon $\mathbf{W}$, we introduce the diagonal degree matrix ${{\mathbf{D}}_{ii}}\!=\!\sum\nolimits_{j=1}^{n}{{{W}_{ij}}}$ and
graph Laplacian matrix $\mathbf{L}=\mathbf{D}-\mathbf{W}$. Besides, we use an $l$-dimensional vector $\mathbf{y}=(y_1,y_2,\cdots,y_l)^{\top}$ to record the given labels of $l$ initial labeled examples, and employ an $n$-dimensional vector $\mathbf{f}=(f_1,\cdots,f_l,f_{l+1},\cdots,f_n)^{\top}$ with $\{f_i\}_{i=1}^{n}\in \mathbb{R}$ to represent the soft labels of all $n$ examples. Furthermore, we define a $|\mathcal{E}|\times n$ ($|\mathcal{E}|$ is the size of edge set $\mathcal{E}$) matrix $\mathbf{P}$, in which the $k$-th ($1\leq k\leq |\mathcal{E}|$) row corresponds to the edge $k$ that connects $\mathbf{x}_i$ and $\mathbf{x}_j$. Specifically, the $k$-th row is defined by
\begin{equation}\label{eq:DefinitonOfP}
\begin{split}
  & {{\mathbf{P}}_{k,:}}=\left( \begin{matrix}
   0 & \cdots  & {{W}_{ij}} & \cdots  & -{{W}_{ij}} & \cdots  & 0  \\
\end{matrix} \right). \\
 & \ \ \ \ \ \ \ \ \ \ \ \ \ \ \ \ \ \ \ \ \ \ \overset{\uparrow }{\mathop{i}}\,\ \ \ \ \ \ \ \ \ \ \ \ \ \ \overset{\uparrow }{\mathop{j}} \\
\end{split}
\end{equation}
Therefore, the \emph{Graph Trend Filtering Term} (\emph{GTF term}) is expressed as $\left\|\mathbf{Pf}\right\|_0=\sum\nolimits_{(i,j)\in\mathcal{E}}W_{ij}\mathds{1}\llbracket f_i\neq f_j\rrbracket$ with ``$\mathds{1}\llbracket \cdot \rrbracket$'' being the indicator function. Consequently, our model is formulated as
\begin{equation}\label{eq:OurModel1}
  \underset{{{\mathbf{f}}=(f_1,\cdots,f_n)^{\top}}}{\mathop{\min }}\ \left\|\mathbf{Pf}\right\|_0+\alpha\left\|\mathbf{Jf}-\mathbf{y}\right\|_0,
\end{equation}
where $\alpha>0$ is the trade-off parameter, and ``$\left\|\cdot\right\|_0$'' represents the $\ell_0$ norm that counts the non-zero elements in the corresponding vector; $\mathbf{J}$ is an $l\times n$ matrix with the $(i,i)$-th ($i=1,2,\cdots,l$) elements being 1, and the other elements being 0. The first \emph{GTF term} in \eqref{eq:OurModel1} enforces the strongly connected examples to obtain identical labels. The second term is \emph{fidelity term} which requires that the optimized $\mathbf{f}$ on the initial labeled examples should approach to the given labels in $\mathbf{y}$. However, the inconsistency between $f_i$ and $y_i$ is allowed due to the adopted $\ell_0$ norm, as not all the elements in $\mathbf{y}$ are correct and trustable.\par

Furthermore, since the Laplacian matrix $\mathbf{L}$ is semi-positive definite, it can be decomposed as $\mathbf{L}=\bar{\mathbf{U}}\bar{\mathbf{\Sigma}}\bar{\mathbf{U}}^{\top}$ where $\bar{\mathbf{\Sigma}}=diag(\lambda_1,\lambda_2,\cdots,\lambda_n)$ is a diagonal matrix with $0=\lambda_1\leq\lambda_2\leq\cdots\leq\lambda_n$ being the totally $n$ eigenvalues, and $\bar{\mathbf{U}}=\left(\mathbf{U}_1,\mathbf{U}_2,\cdots,\mathbf{U}_n\right)$ contains the $n$ associated eigenvectors. Since $\mathbf{U}_1,\mathbf{U}_2,\cdots,\mathbf{U}_n$ are orthogonal, all possible $\mathbf{f}$ on the graph $\mathcal{G}$ can be represented by $\mathbf{f}=\sum\nolimits_{i=1}^{n}{a_i\mathbf{U}_i}$ where $\left\{a_i\right\}_{i=1}^{n}$ are representation coefficients. According to \cite{Goldberg09introductionto,fergus2009semi}, the first $m$ (typically $m\ll n$) eigenvectors usually depict the label smoothness and convey the real class separation, so they can be employed to reconstruct the optimal $\mathbf{f}$ and meanwhile filter out the incorrect initial labels. Therefore, we may write $\mathbf{f}=\mathbf{Ua}$ in \eqref{eq:OurModel1} where $\mathbf{U}=\left(\mathbf{U}_1,\cdots,\mathbf{U}_m\right)$ is the sub-matrix of $\bar{\mathbf{U}}$ containing the first $m$ columns of $\bar{\mathbf{U}}$, and $\mathbf{a}$ is the corresponding coefficient vector, so we have
\begin{equation}\label{eq:OurModel2}
  \underset{\mathbf{a}=(a_1,\cdots,a_m)^{\top}}{\mathop{\min }}\ \left\|\mathbf{PUa}\right\|_0+\alpha\left\|\mathbf{JUa}-\mathbf{y}\right\|_0+\beta\mathbf{a}^{\top}\mathbf{\Sigma}\mathbf{a},
\end{equation}
where ${\mathbf{\Sigma}}$ is a diagonal matrix with ${\mathbf{\Sigma}}=diag(\lambda_1,\cdots,\lambda_m)$, $\alpha,\beta>0$ are two trade-off parameters, and $\mathbf{a}$ is the coefficient vector to be optimized. In \eqref{eq:OurModel2}, the first two terms are directly adapted from \eqref{eq:OurModel1}. The third term allows the coefficient $a_i$ to be large if the corresponding eigenvalue $\lambda_i$ is small, which means that the eigenbasis $\mathbf{U}_i$ with smaller eigenvalues $\lambda_i$ are preferred in the reconstruction of the optimal $\mathbf{f}$, as they are usually smooth and contain rich class information. In contrast, the value of $a_i$ should be suppressed to a small value if the associated eigenvalue $\lambda_i$ is large.\par
Considering that \eqref{eq:OurModel2} involves $\ell_0$ norm that is usually difficult to optimize, we replace the $\ell_0$ norm with the surrogate $\ell_1$ norm, and thus our proposed SIIS model for binary classification is expressed by
\begin{equation}\label{eq:OurModel3}
  \underset{\mathbf{a}=(a_1,\cdots,a_m)^{\top}}{\mathop{\min }}\ \left\|\mathbf{PUa}\right\|_1+\alpha\left\|\mathbf{JUa}-\mathbf{y}\right\|_1+\beta\mathbf{a}^{\top}\mathbf{\Sigma}\mathbf{a},
\end{equation}
in which the $\ell_1$ norm of a vector $\mathbf{h}$ is computed by $\left\|\mathbf{h}\right\|_1=\sum\nolimits_{i}|h_i|$. This model can be easily extended to multi-class cases. Suppose $\mathbf{Y}\in\left\{0,1\right\}^{l\times c}$ ($c$ is the total number of classes) is the label matrix of the initial labeled examples, of which the $i$-th row $\mathbf{Y}_{i,:}$ indicates the label of $\mathbf{x}_i\in\mathcal{L}$. To be specific, $\mathbf{Y}_{ij}=1$ if $\mathbf{x}_i$ belongs to the $j$-th class, and 0 otherwise. As a result, we arrive at the SIIS model for handling multi-class cases, namely:
\begin{equation}\label{eq:OurModel4}
  \underset{\mathbf{A}\in \mathbb{R}^{m\times c}}{\mathop{\min }}\ \left\|\mathbf{PUA}\right\|_{2,1}+\alpha\left\|\mathbf{JUA}-\mathbf{Y}\right\|_{2,1}+\beta \mathrm{tr}\left(\mathbf{A}^{\top}\mathbf{\Sigma}\mathbf{A}\right),
\end{equation}
in which $\left\|\mathbf{H}\right\|_{2,1}$ calculates the $\ell_{2,1}$ norm of the matrix $\mathbf{H}$ by $\left\|\mathbf{H}\right\|_{2,1}=\sum_i\sqrt{\sum_{j}H_{ij}^{2}}$. Therefore, the ``clean'' soft label matrix of all the examples in $\mathcal{L}\cup\mathcal{U}$ can be recovered by $\mathbf{F}=\mathbf{UA}$ in which the $(i,j)$-th element $F_{ij}$ represents the posterior probability of $\mathbf{x}_i$ belonging to the $j$-th class. Consequently, the example $\mathbf{x}_i\in\mathcal{L}\cup\mathcal{U}$ is classified into the $j$-th class if $j\!=\!\arg\max_{j'\in\{1,\cdots,c\}} {F}_{ij'}$.

\section{Optimization}
\label{sec:Optimization}
Without loss of generality, this section introduces the optimization for model \eqref{eq:OurModel4} as it is applicable to multi-class problems. By letting $\mathbf{Q}=\mathbf{PUA}$ and $\mathbf{B}=\mathbf{JUA}-\mathbf{Y}$, \eqref{eq:OurModel4} can be transformed to
\begin{equation}\label{eq:ConstrainedOpt}
\begin{split}
   \underset{\mathbf{A},\mathbf{B},\mathbf{Q}}{\mathop{\min }}\ & \left\|\mathbf{Q}\right\|_{2,1}+\alpha\left\|\mathbf{B}\right\|_{2,1}+\beta \mathrm{tr}\left(\mathbf{A}^{\top}\mathbf{\Sigma}\mathbf{A}\right)\\
   s.t. \ \ & \mathbf{Q}={\mathbf{PUA}},\ \mathbf{B}= \mathbf{JUA}-\mathbf{Y}.\\
\end{split}
\end{equation}
This constrained optimization problem can be easily solved by using the Alternating Direction Method of Multipliers (ADMM), which
alternatively optimizes one variable at one time with the other variables remaining fixed. The augmented Lagrangian function is
\begin{equation}\label{eq:Lagrangian}
\begin{split}
  L(\mathbf{A},\mathbf{B},&\mathbf{Q},\mathbf{\Lambda}_1,\mathbf{\Lambda}_2,\mu) =
   \left\|\mathbf{Q}\right\|_{2,1}\!+\!\alpha\left\|\mathbf{B}\right\|_{2,1}\!+\!\beta \mathrm{tr}\left(\mathbf{A}^{\top}\mathbf{\Sigma}\mathbf{A}\right)\\
   &+\mathrm{tr}\left(\mathbf{\Lambda}_1^{\top}(\mathbf{Q}-\mathbf{PUA})\right)
   \!+\!\mathrm{tr}\left(\mathbf{\Lambda}_2^{\top}(\mathbf{B}\!-\!\mathbf{JUA}\!+\!\mathbf{Y})\right)\\
   &+\frac{\mu}{2}\left(\left\|\mathbf{Q}-\mathbf{PUA}\right\|_{\mathrm{F}}^{2}
   +\left\|\mathbf{B}-\mathbf{JUA}+\mathbf{Y}\right\|_{\mathrm{F}}^{2}\right),\\
\end{split}
\end{equation}
where $\mathbf{\Lambda}_1$ and $\mathbf{\Lambda}_2$ are Lagrangian multipliers, $\mu>0$ is the penalty coefficient, and ``$\left\|\cdot\right\|_{\mathrm{F}}$'' denotes the Frobenius norm of the corresponding matrix. Based on \eqref{eq:Lagrangian}, the variables $\mathbf{A}$, $\mathbf{B}$ and $\mathbf{Q}$ can be sequentially updated via an iterative way.\\
\textbf{Update $\mathbf{Q}$:} The subproblem related to $\mathbf{Q}$ is
\begin{equation}\label{eq:QSubproblem}
\begin{split}
   \underset{\mathbf{Q}}{\mathop{\min }}&\left\|\mathbf{Q}\right\|_{2,1}+\mathrm{tr}\left(\mathbf{\Lambda}_1^{\top}(\mathbf{Q}-\mathbf{PUA})\right)+\frac{\mu}{2}\left\|\mathbf{Q}-\mathbf{PUA}\right\|_{\mathrm{F}}^{2}\\
    \Leftrightarrow & \left\|\mathbf{Q}\right\|_{2,1}+\frac{\mu}{2}\left\|\mathbf{Q}-\mathbf{PUA}+\frac{1}{\mu}\mathbf{\Lambda}_1\right\|_{\mathrm{F}}^2\\
   \Leftrightarrow & \frac{1}{\mu}\left\|\mathbf{Q}\right\|_{2,1}+\frac{1}{2}\left\|\mathbf{Q}-\mathbf{N}\right\|_{\mathrm{F}}^2,\\
\end{split}
\end{equation}
where $\mathbf{N}=\mathbf{PUA}-\frac{1}{\mu}\mathbf{\Lambda}_1$.\par

According to \cite{liu2013robust}, the solution of \eqref{eq:QSubproblem} can be expressed as
\begin{equation}\label{eq:Q_solution}
  \mathbf{Q}_{i,:}=\left\{ \begin{split}
  & \frac{{{\left\| \mathbf{N}_{i,:} \right\|}_{2}}-{1}/{\mu}}{{{\left\| \mathbf{N}_{i,:} \right\|}_{2}}}\mathbf{N}_{i,:},\ \ \ \ {1}/{\mu}<{{\left\| \mathbf{N}_{i,:} \right\|}_{2}} \\
 & \ \ \ \ \ \ \ \ \ \ \  0,\qquad \qquad  \quad \quad \text{otherwise} \\
\end{split} \right.,
\end{equation}
where $\mathbf{N}_{i,:}$ represents the $i$-th row of matrix $\mathbf{N}$.\\
\textbf{Update $\mathbf{B}$:} By dropping the unrelated terms to $\mathbf{B}$ in \eqref{eq:Lagrangian}, the subproblem regarding $\mathbf{B}$ is
\begin{equation}\label{eq:BSubproblem}
   \underset{\mathbf{B}}{\mathop{\min }} \ \alpha\!\left\|\mathbf{B}\right\|_{2,1}+\mathrm{tr}\!\left(\mathbf{\Lambda}_2^{\top}(\mathbf{B}\!-\!\mathbf{JUA}\!+\!\mathbf{Y})\right)
   +\frac{\mu}{2}\!\left\|\mathbf{B}\!-\!\mathbf{JUA}\!+\!\mathbf{Y}\right\|_{\mathrm{F}}^{2},
\end{equation}
which is equivalent to
\begin{equation}\label{eq:BSubproblem2}
   \underset{\mathbf{B}}{\mathop{\min }} \ \frac{\alpha}{\mu}\left\|\mathbf{B}\right\|_{2,1}+\frac{1}{2}\left\|\mathbf{B}-\mathbf{M}\right\|_{\mathrm{F}}^{2},
\end{equation}
where $\mathbf{M}=\mathbf{JUA}-\mathbf{Y}-\frac{1}{\mu}\mathbf{\Lambda}_2$. Similar to \eqref{eq:Q_solution}, the optimizer of \eqref{eq:BSubproblem2} is
\begin{equation}\label{eq:B_solution}
  \mathbf{B}_{i,:}=\left\{ \begin{split}
  & \frac{{{\left\| \mathbf{M}_{i,:} \right\|}_{2}}-{\alpha}/{\mu}}{{{\left\| \mathbf{M}_{i,:} \right\|}_{2}}}\mathbf{M}_{i,:},\ \ \ \ {\alpha}/{\mu}<{{\left\| \mathbf{M}_{i,:} \right\|}_{2}} \\
 & \ \ \ \ \ \ \ \ \ \ \  0,\qquad \qquad  \quad \quad \text{otherwise} \\
\end{split} \right..
\end{equation}\\
\textbf{Update $\mathbf{A}$:} The subproblem regarding $\mathbf{A}$ is
\begin{equation}\label{eq:ASubproblem}
\begin{split}
  \underset{\mathbf{A}}{\mathop{\min }} \
   &\beta \mathrm{tr}\left(\mathbf{A}^{\top}\mathbf{\Sigma}\mathbf{A}\right)-\mathrm{tr}\left(\mathbf{\Lambda}_1^{\top}\mathbf{PUA}\right)-\mathrm{tr}\left(\mathbf{\Lambda}_2^{\top}\mathbf{JUA}\right)\\   &+\frac{\mu}{2}\left(\left\|\mathbf{Q}-\mathbf{PUA}\right\|_{\mathrm{F}}^{2}
   +\left\|\mathbf{B}-\mathbf{JUA}+\mathbf{Y}\right\|_{\mathrm{F}}^{2}\right).\\
\end{split}
\end{equation}
By computing the derivative of \eqref{eq:ASubproblem} w.r.t. $\mathbf{A}$, and then setting the result to zero, we have
\begin{equation}\label{eq:A_solution}
\begin{split}
  \mathbf{A}&=\left(2\beta\mathbf{\Sigma}+\mu\mathbf{U}^{\top}\mathbf{P}^{\top}\mathbf{P}\mathbf{U}+\mu\mathbf{U}^{\top}\mathbf{J}^{\top}\mathbf{J}\mathbf{U}\right)^{-1}\\
  &\left[\mathbf{U}^{\top}\mathbf{P}^{\top}\mathbf{\Lambda}_1\!+\!\mathbf{U}^{\top}\mathbf{J}^{\top}\mathbf{\Lambda}_2
  \!+\!\mu\mathbf{U}^{\top}\mathbf{P}^{\top}\mathbf{Q}\!+\!\mu\mathbf{U}^{\top}\mathbf{J}^{\top}(\mathbf{B}\!+\!\mathbf{Y})\right]\!.\\
\end{split}
\end{equation}\par

The entire proposed SIIS algorithm is summarized in Algorithm~\ref{alg:1}. The related theoretical analyses will be introduced in the next section.

\begin{algorithm}[t]
\small
   \caption{Summary of the proposed algorithm.}
   \label{alg:1}
\begin{algorithmic}[1]
   \STATE {\bfseries Input:} $\alpha$, $\beta$, $m$, $K$, $\mathbf{Y}$.
   \STATE Construct $K$NN graph and compute the adjacency matrix $\mathbf{W}$;

   \STATE Compute the $m$ smallest eigenvalues and eigenvectors of Laplacian matrix $\mathbf{L}$, and store them in $\mathbf{\Sigma}$ and $\mathbf{U}$, respectively;

   \STATE // Begin classification
   \STATE Set $\mathbf{\Lambda}_1, \mathbf{\Lambda}_2$ to all-one matrices; Initialize $\mathbf{A}=\mathbf{O}$; Set $\mu\!=\!1$, $\mu_{max}\!=\!10^{10}$, $\rho\!=\!1.2$, $\epsilon=10^{-4}$, $MaxIter=100$;
   \STATE set $iter\!=\!0$;
   \REPEAT
   \STATE Update $\mathbf{Q}$ via \eqref{eq:Q_solution};
   \STATE Update $\mathbf{B}$ via \eqref{eq:B_solution};
   \STATE Update $\mathbf{A}$ via \eqref{eq:A_solution};
   \STATE // Update Lagrangian multipliers
   \STATE $\mathbf{\Lambda}_1:=\mathbf{\Lambda}_1+\mu(\mathbf{Q}-\mathbf{PUA})$,
   $\mathbf{\Lambda}_2:=\mathbf{\Lambda}_2+\mu(\mathbf{B}-\mathbf{JUA}+\mathbf{Y})$;

   \STATE // Update penalty coefficients
   \STATE $\mu:=\min(\rho\mu, \mu_{max})$;
   \STATE $iter:=iter+1$;
   \UNTIL{$\frac{\left\|\mathbf{A}^{(iter)}-\mathbf{A}^{(iter-1)}\right\|_{\infty}}{\left\|\mathbf{A}^{(iter-1)}\right\|_{\infty}}\leq\epsilon$ or $iter=MaxIter$}
   \STATE Recover $\mathbf{F}=\mathbf{UA}$;
   \STATE Classify $\mathbf{x}_i\in\mathcal{L}\cup\mathcal{U}$ to the $j$-th class via $j\!=\!\arg\max_{j'\in\{1,\cdots,c\}} {F}_{ij'}$;
   \STATE {\bfseries Output:} Class labels $\{y_i\}_{i=1}^{n}$.
\end{algorithmic}
\end{algorithm}

\section{Theoretical Analyses}
\label{sec:TheoreticalAnalyses}
In this section, we firstly prove that the optimization process explained in Section~\ref{sec:Optimization} will converge to a stationary point, and then analyse the computational complexity of the proposed SIIS method.

\subsection{Proof of Convergence}
\label{sec:Convergence}

Up to now, some prior works such as \cite{eckstein1992douglas,Liu2014Enhancing} have provided the sufficient conditions for the convergence of a general ADMM solver, which is

\begin{theorem} \label{thm:GeneralConvergence}
\textbf{\cite{eckstein1992douglas,Liu2014Enhancing}}
Given the optimization problem with linear constraints as:
\begin{equation}\label{eq:olc}
   {\min}_{\{X,V\}} f(X)+g(V), \ \ \ s.t. \ \ \mathbf{G}X=V,
\end{equation}
where $\mathbf{G}$ is the coefficient matrix, and $f(X), g(V)$ are two functions w.r.t. the optimizers $X$ and $V$, respectively. The augmented Lagrangian function of \eqref{eq:olc} is
\begin{equation}\label{eq:alg}
\begin{split}
  L&=f(X)+g(V)+\langle Z,\mathbf{G}X-V\rangle+\frac{\mu}{2}\left\|\mathbf{G}X-V\right\|^2_\mathrm{F}\\
   &= f(X)+g(V)+\frac{\mu}{2}\left\|\mathbf{G}X-V+C\right\|^2_\mathrm{F}+\textmd{constant},\\
\end{split}
\end{equation}
where $Z$ is the Lagrange multiplier, $\mu>0$ is the penalty coefficient, and $C=\frac{Z}{\mu}$. Then, if $f(X)$ and $g(V)$ are convex and $\mathbf{G}$ has full column rank, there exists a solution of problem \eqref{eq:alg} such that the sequences $\{X_t, V_t, C_t\}$ generated by the ADMM algorithm will converge given the initial values $X_0$, $V_0$ and $C_0=\frac{Z_0}{\mu_0}$. Otherwise, at least one of the sequences $\{X_t, V_t\}$ and $\{C_t\}$ will diverge.
\end{theorem}

Before verifying the convergence of our SIIS algorithm, we first provide some useful lemmas.

\begin{lemma}\label{lemma:HHT}
\textbf{\cite{petersen2008matrix}} For any real matrix $\mathbf{H}$, it holds that $r(\mathbf{H}^{\top}\mathbf{H})=r(\mathbf{H}\mathbf{H}^{\top})=r(\mathbf{H})$
\end{lemma}

\begin{lemma}\label{lemma:MultiplicationRank}
Given three matrices $\mathbf{H}_{m\times n}$, $\mathbf{R}_{n\times q}$ and $\mathbf{T}_{m\times q}$ that satisfy $\mathbf{HR}=\mathbf{T}$, and $r(\mathbf{H})=n$ with ``$r(\cdot)$'' denoting the rank of the corresponding matrix, then we have $r(\mathbf{R})=r(\mathbf{T})$.
\end{lemma}
\begin{proof}
Since $r(\mathbf{H})=n$, we have an invertible matrix $\mathbf{P}_e$ such that $\mathbf{P}_{e}\mathbf{H}=\left(\begin{smallmatrix}
{\mathbf{I}_{n}} \\ \mathbf{O}  \\ \end{smallmatrix} \right)$ where $\mathbf{I}_n$ is the $n\times n$ identity matrix. Therefore, we have
\begin{equation}\label{eq:ProveLemma1}
\left( \begin{matrix}
   \mathbf{R}  \\
   \mathbf{O}  \\
\end{matrix} \right)=\left( \begin{matrix}
   {\mathbf{I}_{n}}  \\
   \mathbf{O}  \\
\end{matrix} \right)\mathbf{R}=\mathbf{P}_{e}\mathbf{HR}=\mathbf{P}_{e}\mathbf{T}.
\end{equation}
Since $\mathbf{P}_e$ is invertible, we know that $r(\mathbf{P}_{e}\mathbf{T})=r(\mathbf{T})$. By further noticing that $r\left( \begin{smallmatrix}    \mathbf{R}  \\    \mathbf{O}  \\ \end{smallmatrix} \right)=r(\mathbf{R})$, we arrive at $r(\mathbf{R})=r(\mathbf{T})$.
\end{proof}

Now we formally present the theorem that guarantees the convergence of our SIIS algorithm:
\begin{theorem}\label{thm:ConvergeOfOurMethod}
Given the optimization problem \eqref{eq:ConstrainedOpt}, the iterative ADMM process will converge to a stationary point.
\end{theorem}
\begin{proof}
To facilitate the proof, we first rewrite the optimization problem \eqref{eq:ConstrainedOpt} as the formation of \eqref{eq:olc}. It can be easily verified that if we set $X=\textbf{A}$, $V = \left( \begin{smallmatrix}   \mathbf{Q}  \\ \mathbf{B}+\mathbf{Y}  \\ \end{smallmatrix} \right)$, and $\mathbf{G} = \left( \begin{smallmatrix}   \mathbf{PU}  \\ \mathbf{JU}  \\ \end{smallmatrix} \right)$, the $f(X)$ and $g(V)$ in \eqref{eq:olc} can be expressed by $f(X)=\beta \mathrm{tr}\left(\mathbf{A}^{\top}\mathbf{\Sigma}\mathbf{A}\right)$ and $g(V)=\left\|\mathbf{Q}\right\|_{2,1}+\alpha\left\|\mathbf{B}\right\|_{2,1}$, respectively. Therefore, both $f(X)$ and $g(V)$ in our case are convex. Next, according to Theorem~\ref{thm:GeneralConvergence}, we have to verify that the matrix $\mathbf{G}$ has full column rank.\par

By formulating $\mathbf{G} = \left( \begin{smallmatrix}   \mathbf{PU}  \\ \mathbf{JU}  \\ \end{smallmatrix} \right)=\left( \begin{smallmatrix}   \mathbf{P}  \\   \mathbf{J}  \\ \end{smallmatrix} \right)\mathbf{U}=\mathbf{\Omega}\mathbf{U}$ where $\mathbf{\Omega}=\left( \begin{smallmatrix}   \mathbf{P}  \\   \mathbf{J}  \\ \end{smallmatrix} \right)$, we investigate the rank of $\mathbf{\Omega}$. It is straightforward that $\mathbf{P}$ has the identical rank with the matrix $\tilde{\mathbf{P}}=sgn(\mathbf{P})$, in which $sgn(\mathbf{P})$ returns $\tilde{{P}}_{ij}=1,0,-1$ if the corresponding $P_{ij}$ is positive, 0, and negative, respectively. Furthermore, by noticing that $\tilde{\mathbf{P}}^{\top}\tilde{\mathbf{P}}=sgn({\mathbf{L}})$ \cite{wang2016trend} with $\mathbf{L}$ being the graph Laplacian matrix, and $r(sgn({\mathbf{L}}))=n-1$ given that the graph $\mathcal{G}$ is connected \cite{Goldberg09introductionto}, we know that $r(\mathbf{P})=r(\tilde{\mathbf{P}})=r(sgn({\mathbf{L}}))=n-1$ according to Lemma~\ref{lemma:HHT}. Since $\mathbf{P}$ is a submatrix of $\mathbf{\Omega}$, we obtain $r(\mathbf{\Omega})\geq n-1$. On the other hand, since $\mathbf{J}$ has the form $\mathbf{J}=\left(\begin{matrix}\mathbf{I}_l & \mathbf{O} \\ \end{matrix}\right)$, of which the rows are not in the row space of $\mathbf{P}$, we know that $\mathbf{\Omega}$ has full column rank $n$. Consequently, according to Lemma~\ref{lemma:MultiplicationRank} we arrive at $r(\mathbf{G})=r(\mathbf{U})$. Since the columns of $\mathbf{U}$ correspond to the eigenvectors of $\mathbf{L}$ that are orthogonal, we know $r(\mathbf{G})=r(\mathbf{U})=m$, which means that $\mathbf{G}\in\mathbb{R}^{(|\mathcal{E}|+l)\!\times\! m}$ has full column rank, and thus our method is guaranteed to converge.

%

\end{proof}
\vspace{-18pt}

\subsection{Computational Complexity}
\label{sec:Complexity}

This section studies the computational complexity of SIIS. The graph construction in Line~2 of Algorithm~\ref{alg:1} takes $\mathcal{O}(n^2)$ complexity. The Line~3 is accomplished by using the Implicit Restarted Lanczos Method \cite{calvetti1994implicitly}, of which the complexity is $\mathcal{O}\left((n_0+n)t_1\right)$, where $n_0$ is the number of non-zero elements in $\mathbf{L}$ and $t_1$ is the number of iterations required until convergence. In Line~8, one should compute the $\ell_2$ norm of every row of a $|\mathcal{E}|\times n$ matrix $\mathbf{N}$, so its complexity is $\mathcal{O}(|\mathcal{E}|n)$. Similarly, the complexity of Line~9 is $\mathcal{O}(lc)$. Note that a $m\times m$ matrix should be inverted in Line~10, so the complexity of this step is $\mathcal{O}(m^3)$. As the complexity of Line~17 is $\mathcal{O}(nmc)$, the total complexity of our SIIS algorithm is $\mathcal{O}\big((|\mathcal{E}|c+lc+m^3)t_2+n^2+(n_0+n)t_1+nmc\big)$ by assuming that Lines~7$\sim$16 are iterated $t_2$ times. Although this complexity is cubic to $m$, this parameter is practically set to a small positive value, so the complexity of our method is acceptable.

\section{Experimental Results}
\label{sec:experiments}

In this section, we first validate the motivation of the proposed algorithm (Section~\ref{sec:Validation}), and then compare our SIIS with several representative approaches on various practical datasets related to image, text, and speech (Sections~\ref{sec:ImageData}$\sim$\ref{sec:AudioData}). Finally, we empirically examine the convergence property of optimization (Section~\ref{sec:IllustrationOfConvergence}) and study the parametric sensitivity of SIIS (Section~\ref{sec:ParametricSensitivity}).

\subsection{Algorithm Validation}
\label{sec:Validation}
\begin{figure*}
  \centering
  \includegraphics[width=\linewidth]{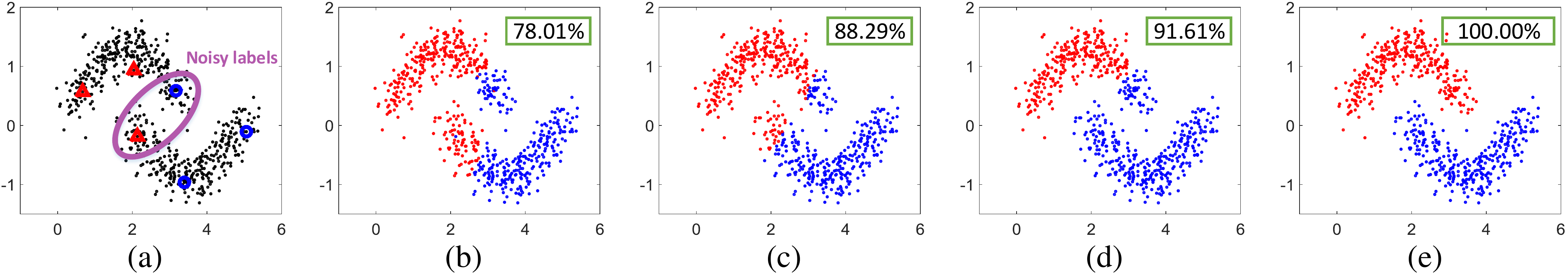}\\
  \caption{Algorithm validation on \emph{NoisyDoubleMoon} dataset. (a) shows the initial state with labeled positive examples (red triangles) and labeled negative examples (blue circles). Note that the labels of two labeled examples are incorrect and they form the noisy labels. (b)$\sim$(e) present the classification results of different settings.}\label{fig:NoisyDoubleMoon}
  \vskip -12pt
\end{figure*}

There are three critical components in our model \eqref{eq:OurModel3} for tackling the possible label noise: 1) $\ell_{1}$ norm is adopted to form the fidelity term, allowing the obtained solution $\mathbf{f}$ to be slightly inconsistent with the initial given labels in $\mathbf{y}$; 2) GTF term is employed to adaptively correct the possible label error in a local area, which is better than the global smoothness term $\mathrm{tr}(\mathbf{f}^{\top}\mathbf{Lf})$ that has been widely used by many existing methodologies; and 3) The SEP strategy is leveraged to recover the precise labels by emphasizing the $\mathbf{L}$'s leading eigenvectors with clear class indication. Therefore, here we use a two-dimensional toy dataset, i.e. \emph{NoisyDoubleMoon}, to visually show the strength of each of above three components.\par

\emph{NoisyDoubleMoon} consists of 640 examples, which are equally divided into two moons. This dataset is contaminated by the Gaussian noise with standard deviation 0.15, and each class has three initial labeled examples (see Fig.~\ref{fig:NoisyDoubleMoon}(a)). However, each class contains one erroneously labeled example (marked by the purple circle), which poses a great difficulty for an algorithm to achieve perfect classification.\par

Fig.~\ref{fig:NoisyDoubleMoon}(b) presents the result of Gaussian Field and Harmonic Functions (GFHF) \cite{Zhu03semisupervisedlearning}, of which the model is ${\mathop{\min }}_{\mathbf{f}:\mathbf{f}_{\mathcal{L}}=\mathbf{y}}\ \mathrm{tr}(\mathbf{f}^{\top}\mathbf{Lf})$. As GFHF requires the finally obtained $\mathbf{f}$ to be strictly identical to $\mathbf{y}$ on labeled set $\mathcal{L}$ and completely ignores the label noise, we see that it is greatly misled by the two incorrect labels and only obtains $78.01\%$ accuracy. In (c), we replace the equality constraint $\mathbf{f}_{\mathcal{L}}=\mathbf{y}$ in GFHF with a robust $\ell_{1}$ fidelity term, and the model is ${\mathop{\min }}_{\mathbf{f}}\ \mathrm{tr}(\mathbf{f}^{\top}\mathbf{Lf})+\alpha\left\|\mathbf{Jf}-\mathbf{y}\right\|_{1}$. We observe that this $\ell_{1}$ fidelity term greatly weakens the negative effect of mislabeled examples, and the performance can be improved to $88.29\%$, therefore the effectiveness of component 1) is verified. In (d), we further replace the traditional Laplacian smoother $\mathrm{tr}(\mathbf{f}^{\top}\mathbf{Lf})$ in (c) with GTF term to form the model ${\mathop{\min }}_{\mathbf{f}}\ \left\|\mathbf{Pf}\right\|_{1}+\alpha\left\|\mathbf{Jf}-\mathbf{y}\right\|_{1}$, and investigate the effect brought by GTF. It can be clearly observed that the false-positive labeled data point in the below moon has been corrected, and thus all negative examples have been successfully identified. Consequently, the classification accuracy has been improved to $91.61\%$, which means that GTF is helpful for eliminating the label noise. Finally, we illustrate the classification result yielded by the proposed model in (e). By preserving the top-2 eigenvectors of $\mathbf{L}$ to reconstruct $\mathbf{f}$, all examples are accurately classified, therefore component 3) contributes to enhance the robustness of our SIIS method. In short, every step included by SIIS is helpful for suppressing the adverse effect of noisy labels.


\subsection{Image Data}
\label{sec:ImageData}

\begin{figure}[t]
  \centering
  \includegraphics[width=0.9\linewidth]{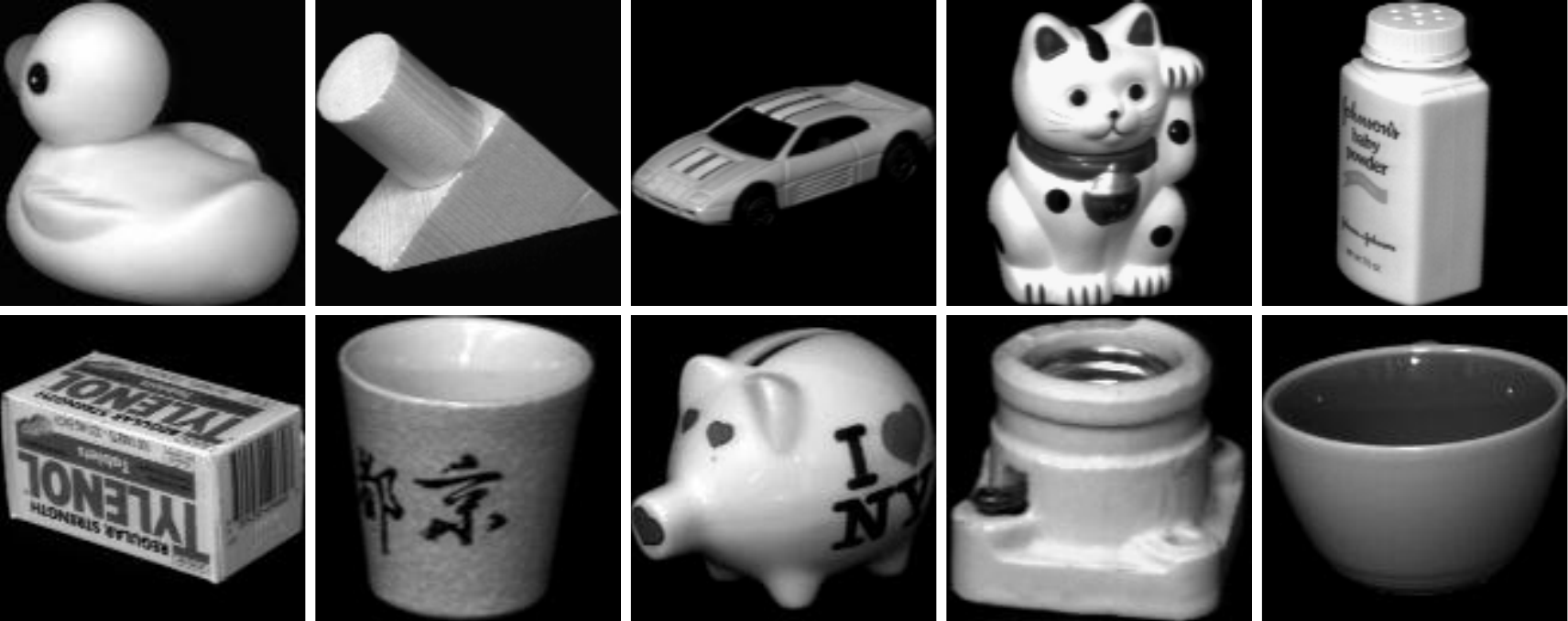}\\
  \caption{Example images of \emph{COIL} dataset.}\label{fig:ExamplesCOIL20}
  \vskip -10pt
\end{figure}

We firstly use the \emph{COIL} dataset \cite{nene1996columbia} to test the ability of our SIIS on processing image data. \emph{COIL} is a popular public dataset for object recognition which contains 1440 object images belonging to 20 classes (see Fig.~\ref{fig:ExamplesCOIL20} for some examples), and each object has 72 images shot from different angles. The resolution of each image is $32\times 32$, with 256 grey levels per pixel. Thus, every image is represented by a 1024-dimensional element-wise vector. We randomly select 10 examples from each class to establish the labeled set, and then the remaining image examples form the unlabeled set. To incorporate different levels of label noise to the dataset, we randomly pick up $0\%$, $20\%$, $40\%$ and $60\%$ examples from the totally $10\!\times\!20\!=\!200$ labeled examples, and switch the correct label of each of them to a random wrong label. Such label contamination is conducted 10 times, so every compared algorithm should independently run 10 times on the contaminated dataset and the average accuracy over these 10 different runs are particularly investigated.\par
The compared algorithms include: 1) the traditional supervised classifier SVM; 2) typical SSL method GFHF \cite{Zhu03semisupervisedlearning} which utilizes Laplacian smoother and does not consider the label noise, and 3) state-of-the-art label-noise robust SSL methodologies such as Large-Scale Sparse Coding (LSSC) \cite{lu2015noise}, Graph Trend Filtering (GTF) \cite{wang2016trend}, and Self-Paced Manifold Regularization (SPMR) \cite{gu2016robust}.\par

For fair comparison, all graph-based algorithms such as GFHF, LSSC, GTF, SPMR and SIIS are implemented on a 10-NN graph with Gaussian kernel width $\xi=100$. In SIIS, $\alpha$ and $\beta$ are set to 100 and 10, respectively, and the number of preserved eigenvectors is set to $m=30$. Similarly, the trade-off parameter for fidelity term in GTF is also tuned to 100, and the first 30 eigenvectors are employed to reconstruct the labels in LSSC. In SPMR, the parameters are tuned to $\gamma_{K}=1$ and $\gamma_{I}=0.01$ as suggested by \cite{gu2016robust}.\par

The classification accuracies of all compared methods on labeled set and unlabeled set are presented in Figs.~\ref{fig:Accuracy_COIL}(a) and (b), respectively. It can be observed that the performances of all algorithms decrease with the increase of label noise level. However, the proposed SIIS achieves the best results in most cases when compared with other baselines. Another notable fact is that SVM, which is a traditional supervised algorithm, performs worse than any of the SSL methods, therefore SSL adopted in this paper is more effective than supervised models when the supervision information is inadequate. When it comes to SSL approaches, we see that the performance of GFHF decreases dramatically when the noise level ranges from $0\%$ to $60\%$, which suggests that the $\ell_2$ norm-based Laplacian smoother would fail in the presence of contaminated labels. In contrast, some robust SSL methods including LSSC, GTF and our SIIS are very stable although the noise level increases rapidly, especially in the range $\left[0\%, 40\%\right]$. This indicates that although the very limited supervision information in SSL might be inaccurate, they can be well dealt with if some specific measures are taken.

\begin{figure}[t]
  \centering
  \includegraphics[width=\linewidth]{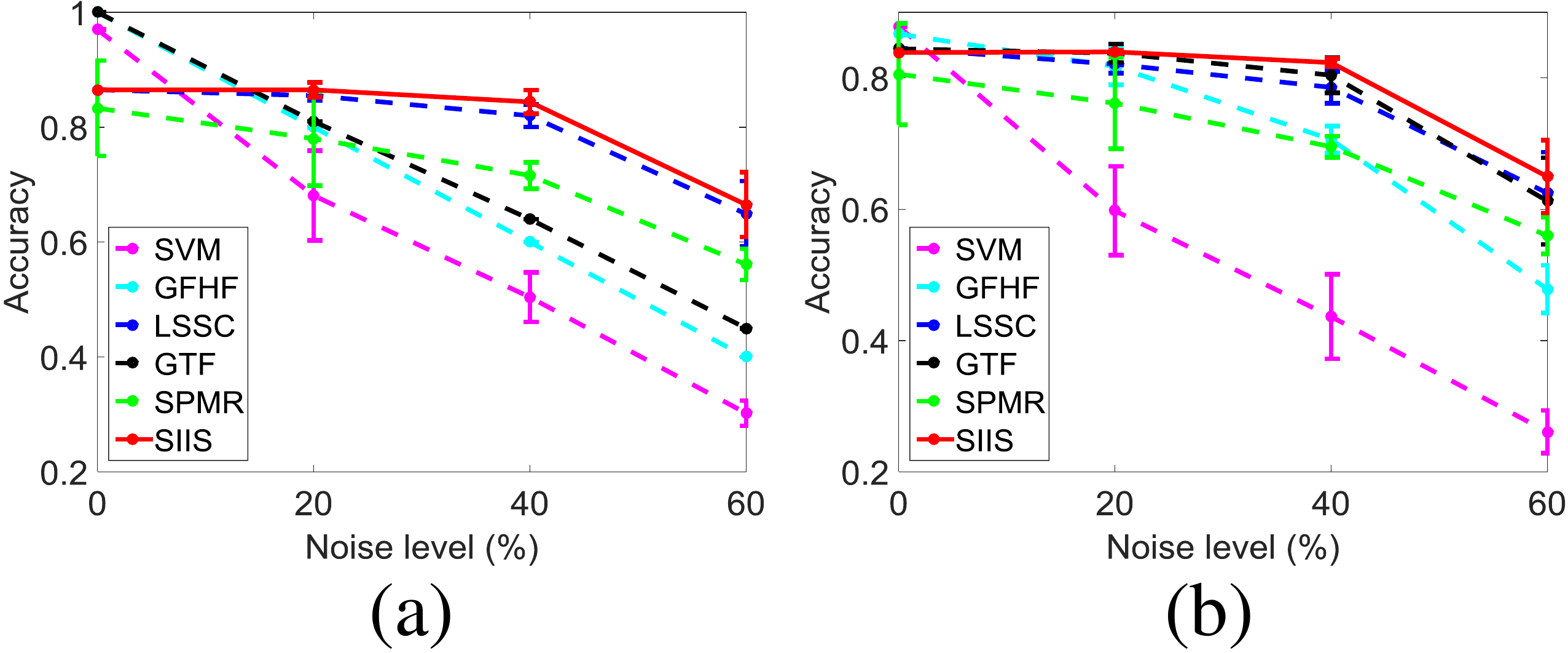}\\
  \caption{The comparison of various algorithms on \emph{COIL} dataset. (a) shows their classification accuracies on labeled set, and (b) plots the accuracies on unlabeled set.}\label{fig:Accuracy_COIL}
   \vskip -7pt
\end{figure}

\subsection{Text Data}
\label{sec:TextData}
This section compares the ability of SVM, GFHF, LSSC, GTF, SPMR, and the proposed SIIS on text categorization. Similar to \cite{cai2012manifold}, a subset of the \emph{Reuters Corpus Volume 1} (\emph{RCV1}) dataset \cite{lewis2004rcv1} is adopted for comparison, which contains 9625 news article examples 
across four classes (i.e. ``C15'', ``ECAT'', ``GCAT'', and ``MCAT''). As there are 29992 distinct words in this dataset, the standard TF-IDF weighting scheme is adopted to generate a 29992-dimensional feature vector for each example.\par

In this dataset, we randomly select 240 examples out of the totally 9625 examples to establish the labeled set, and the remaining 9385 examples are regarded as unlabeled. As a result, the labeled examples only accounts for approximately $2.5\%$ of the entire dataset, leading to the scarcity of the supervision information. Similar to Section~\ref{sec:ImageData}, here we also arbitrarily select $0\%\sim60\%$ labeled examples and then manually set their labels to the wrong ones. We are interested in the performances of compared methods under different levels of label noise.\par

The parameters of established graph are $K=10$ and $\xi=10$. The free parameters $\alpha$ and $\beta$ are tuned via searching the grid $\left[1,10,100,1000 \right]$, and their values are determined as $\alpha\!=\!1000$ and $\beta\!=\!10$. In Section~\ref{sec:ParametricSensitivity}, we will explain the reason for choosing such parametric setting. The performances rendered by the compared methods are displayed in Fig.~\ref{fig:Accuracy_RCV1}, in which (a) presents the classification accuracy on labeled examples and (b) shows the accuracy on unlabeled examples. From (a), we see that SVM, GFHF and GTF obtain almost $100\%$ accuracy on the labeled set $\mathcal{L}$ when there is no label noise. However, their accuracies decrease dramatically when we gradually increase the noise level, and they are worse than SIIS under heavy label noise in the range $\left[40\%, 60\%\right]$. As a consequence, they are inferior to SIIS in terms of the classification accuracy on unlabeled set $\mathcal{U}$ as reflected by (b). Besides, by comparing the accuracies of SIIS and GFHF when label noise presents, we see that SIIS leads GFHF with a noticeable margin no matter on the labeled set or unlabeled set. This demonstrates that SIIS with GTF term and SEP term is better than the GFHF with $\ell_2$ norm-based Laplacian smoother on amending the label errors. Furthermore, by comparing SIIS and GTF, we note that the performance of GTF can be remarkably improved by SIIS, and this again validates that the SEP in our SIIS plays an important role in filtering out the noisy labels. These observations are also consistent with our previous findings in Section~\ref{sec:Validation}.

\begin{figure}
  \centering
  \includegraphics[width=\linewidth]{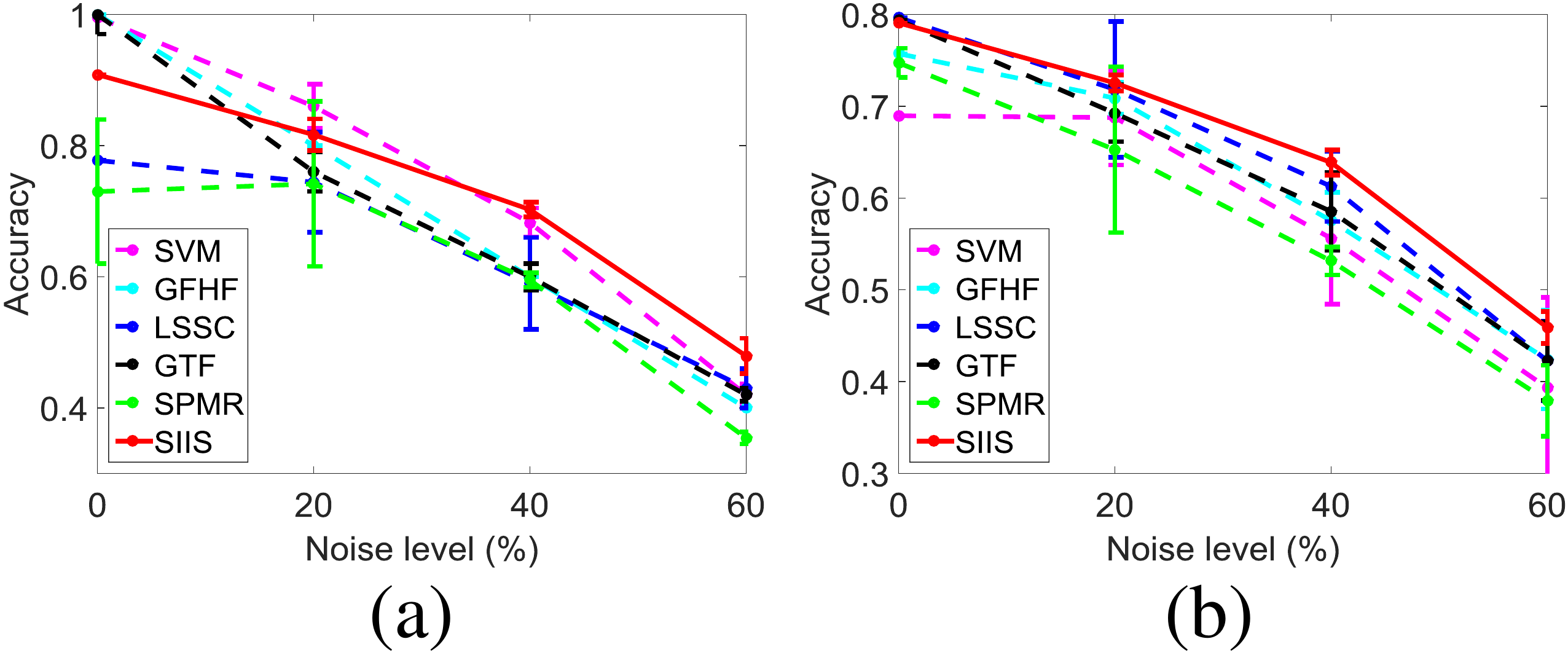}\\
  \caption{The comparison of various algorithms on \emph{RCV1} dataset. (a) shows their classification accuracies on labeled set, and (b) displays the accuracies on unlabeled set.}\label{fig:Accuracy_RCV1}
   \vskip -8pt
\end{figure}

\subsection{Audio Data}
\label{sec:AudioData}
In this experiment, we address a speech recognition task by using the \emph{ISOLET} dataset\footnote{\url{https://archive.ics.uci.edu/ml/datasets/ISOLET}}. In this dataset, 150 subjects are required to pronounce each letter in the alphabet (i.e. ``A''$\sim$``Z'') twice. Excluding 3 missing examples, we have totally $150\times2\times26-3=7797$ examples. Our task is to identify which of the 26 letters every example belongs to. Among the 7797 examples, we extract 40 examples from each class to form the labeled set with size 1040, and the rest 6757 examples are treated as unlabeled.\par

We vary the label noise level from 0\% to 60\%, and investigate the classification performances of SVM, GFHF, LSSC, GTF, SPMR and SIIS on labeled and unlabeled examples. From Tables~\ref{tab:ISOLET_Labeled} and \ref{tab:ISOLET_Unlabeled}, we see that SIIS performs comparably with GFHF and LSSC when the noise level is not larger than 40\%. However, if 60\% labeled examples are erroneously annotated, the advantage of SIIS becomes prominent among all compared algorithms. Specifically, SIIS reaches 77.4\% accuracy on $\mathcal{L}$ and 74.9\% accuracy on $\mathcal{U}$, respectively, which leads the second best algorithm LSSC with a gap 5.6\% and 7.2\% correspondingly. Besides, it can be found that the error rates of GFHF on $\mathcal{L}$ are equivalent to the noise rates such as $20\%$, $40\%$ and $60\%$, as this method requires the labels of original labeled examples to remain unchanged during the classification. In contrast, our SIIS generates 86.4\%, 83.6\% and 77.4\% accuracy on $\mathcal{L}$ when 20\%, 40\% and 60\% labels are not correct, which means that 6.4\%, 23.6\% and 37.4\% deteriorated labels have been corrected correspondingly, and this is the reason that our method is able to obtain satisfactory performance on classifying the unlabeled examples in the presence of heavy label noise. Specifically, SIIS touches 74.9\% accuracy on unlabeled examples although more than half (i.e. 60\%) of the available labels are inaccurate, which is an impressive result when compared with other existing methodologies.

\begin{table*}[t]\small
\setlength\tabcolsep{15pt}
\caption{The comparison of various methods on \emph{ISOLET} dataset. The classification accuracies on labeled examples are presented. The best record under each label noise level is marked in bold.} \label{tab:ISOLET_Labeled}
\begin{center}
\begin{tabular}{|c|c|c|c|c|}
\hline
\backslashbox{Methods\kern-2em}{Noise level\kern-1.3em} & $0\%$ &  $20\%$ & $40\%$ & $60\%$ \\
\hline\hline
SVM          & 0.943 $\pm$ 0.000 & 0.813 $\pm$ 0.012 & 0.687 $\pm$ 0.022 & 0.521 $\pm$ 0.015 \\
GFHF         & \textbf{1.000 $\pm$ 0.000} & 0.800 $\pm$ 0.000 & 0.600 $\pm$ 0.000 & 0.400 $\pm$ 0.000 \\
LSSC         & 0.899 $\pm$ 0.000 & 0.877 $\pm$ 0.003 & 0.829 $\pm$ 0.009 & 0.718 $\pm$ 0.017 \\
GTF          & 0.958 $\pm$ 0.000 & 0.798 $\pm$ 0.007 & 0.633 $\pm$ 0.004 & 0.553 $\pm$ 0.006 \\
SPMR         & 0.685 $\pm$ 0.000 & 0.638 $\pm$ 0.008 & 0.635 $\pm$ 0.004 & 0.548 $\pm$ 0.005 \\
SIIS         & 0.911 $\pm$ 0.000 & \textbf{0.905 $\pm$ 0.008} & \textbf{0.836 $\pm$ 0.010} & \textbf{0.774 $\pm$ 0.010} \\
\hline
\end{tabular}
\end{center}
\vskip -7pt
\end{table*}

\begin{table*}[t]\small
\setlength\tabcolsep{15pt}
\caption{The comparison of various methods on \emph{ISOLET} dataset. The classification accuracies on unlabeled examples are presented. The best record under each label noise level is marked in bold.} \label{tab:ISOLET_Unlabeled}
\begin{center}
\begin{tabular}{|c|c|c|c|c|}
\hline
\backslashbox{Methods\kern-2em}{Noise level\kern-1.3em} & $0\%$ &  $20\%$ & $40\%$ & $60\%$ \\
\hline\hline
SVM          & 0.851 $\pm$ 0.000 & 0.805 $\pm$ 0.011 & 0.729 $\pm$ 0.021 & 0.594 $\pm$ 0.017 \\
GFHF         & \textbf{0.865 $\pm$ 0.000} & 0.816 $\pm$ 0.004 & 0.797 $\pm$ 0.010 & 0.674 $\pm$ 0.015 \\
LSSC         & 0.848 $\pm$ 0.000 & 0.828 $\pm$ 0.003 & 0.785 $\pm$ 0.006 & 0.677 $\pm$ 0.018 \\
GTF          & 0.701 $\pm$ 0.000 & 0.699 $\pm$ 0.002 & 0.598 $\pm$ 0.003 & 0.548 $\pm$ 0.005 \\
SPMR         & 0.689 $\pm$ 0.000 & 0.638 $\pm$ 0.008 & 0.627 $\pm$ 0.004 & 0.539 $\pm$ 0.005 \\
SIIS         & 0.854 $\pm$ 0.000 & \textbf{0.849 $\pm$ 0.006} & \textbf{0.802 $\pm$ 0.013} & \textbf{0.749 $\pm$ 0.014} \\
\hline
\end{tabular}
\end{center}
\vskip -10pt
\end{table*}

\subsection{Illustration of Convergence}
\label{sec:IllustrationOfConvergence}

\begin{figure}
  \centering
  \includegraphics[width=\linewidth]{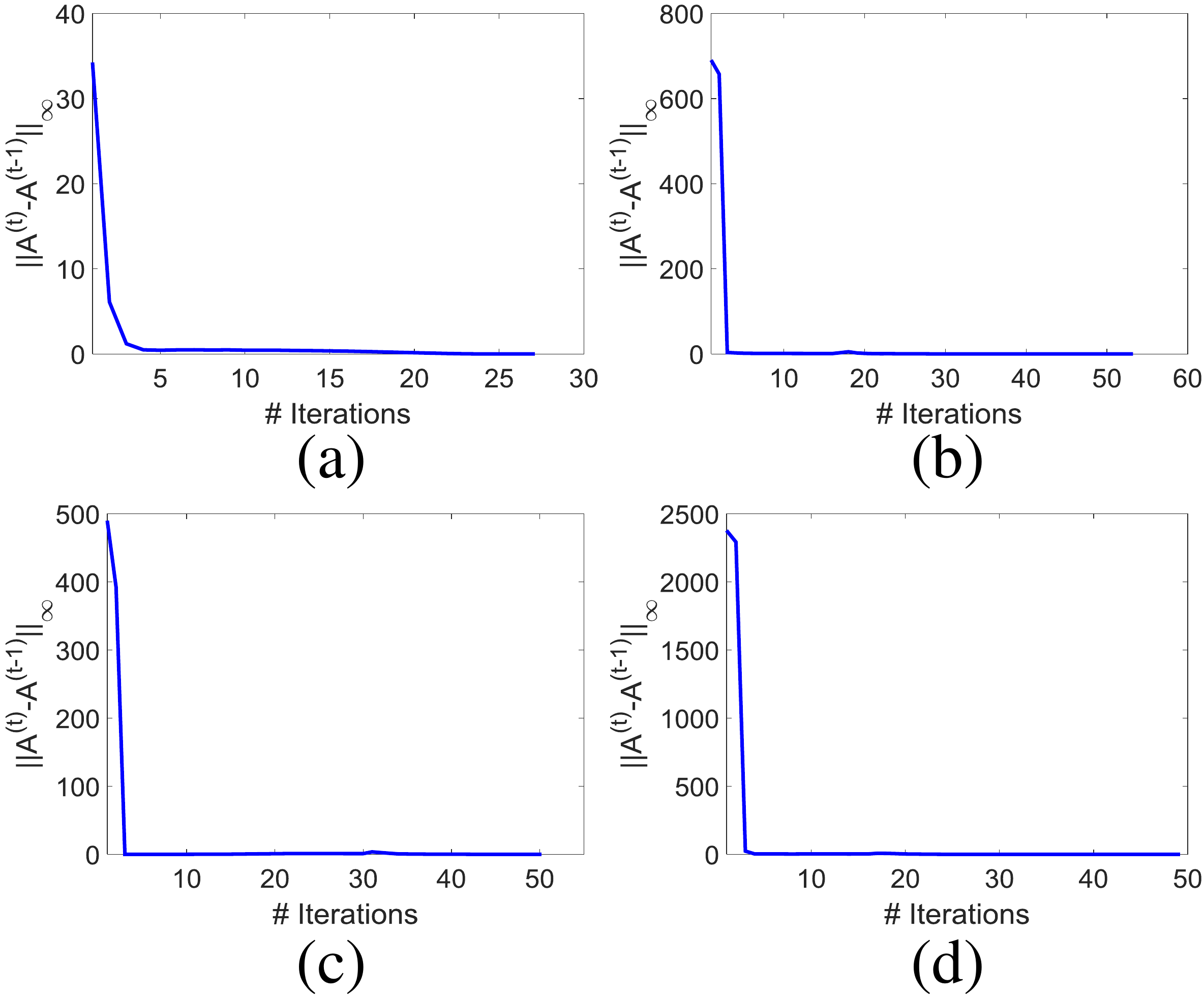}\\
  \caption{The convergence process of the ADMM method adopted by our SIIS algorithm. (a) is \emph{NoisyDoubleMoon} dataset, (b) is \emph{COIL} dataset, (c) is \emph{RCV1} dataset, and (d) is \emph{ISOLET} dataset.}\label{fig:ConvergenceCurve}
  \vskip -12pt
\end{figure}

In Section~\ref{sec:Convergence}, we have theoretically proved that the optimization process in our method will converge to a stationary point. In this section, we present the convergence curves of SIIS on the four datasets appeared from Section~\ref{sec:Validation} to Section~\ref{sec:AudioData}. From the curves in Fig.~\ref{fig:ConvergenceCurve}, we see that the difference between the successive $\mathbf{A}$ decreases rapidly on all four datasets. This observation justifies our previous theoretical results and demonstrates that ADMM is effective for solving the SIIS model in \eqref{eq:OurModel4}.

\subsection{Parametric Sensitivity}
\label{sec:ParametricSensitivity}
Note that the objective function \eqref{eq:OurModel4} in our method contains two trade-off parameters $\alpha$ and $\beta$ that should be manually tuned. Therefore, in this section we discuss whether the choices of them will significantly influence the performance of SIIS. To this end, we examine the classification accuracy of SIIS on unlabeled set by varying one of $\alpha$ and $\beta$, and meanwhile fixing the other one to a constant value. The three practical datasets from Sections~\ref{sec:ImageData} to \ref{sec:AudioData} are adopted including \emph{COIL}, \emph{RCV1}, and \emph{ISOLET}. By changing $\alpha$ and $\beta$ from $10^0$ to $10^3$, the results under 40\% label noise on the three datasets are shown in Fig.~\ref{fig:ParametricSensitivity}. From the experimental results, we learn that the these two parameters are critical for our algorithm to achieve good performance. To be specific, $\alpha$ is suggested to choose a relatively large number, such as 100 on \emph{COIL} and \emph{ISOLET} datasets, and 1000 on \emph{RCV1} dataset. According to Fig.~\ref{fig:ParametricSensitivity}(b), we see that a small $\beta$ is preferred to obtain high accuracy, therefore this parameter is set to 10, 10, and 1 on \emph{COIL}, \emph{RCV1} and \emph{ISOLET}, respectively.

\begin{figure}[t]
  \centering
  \includegraphics[width=\linewidth]{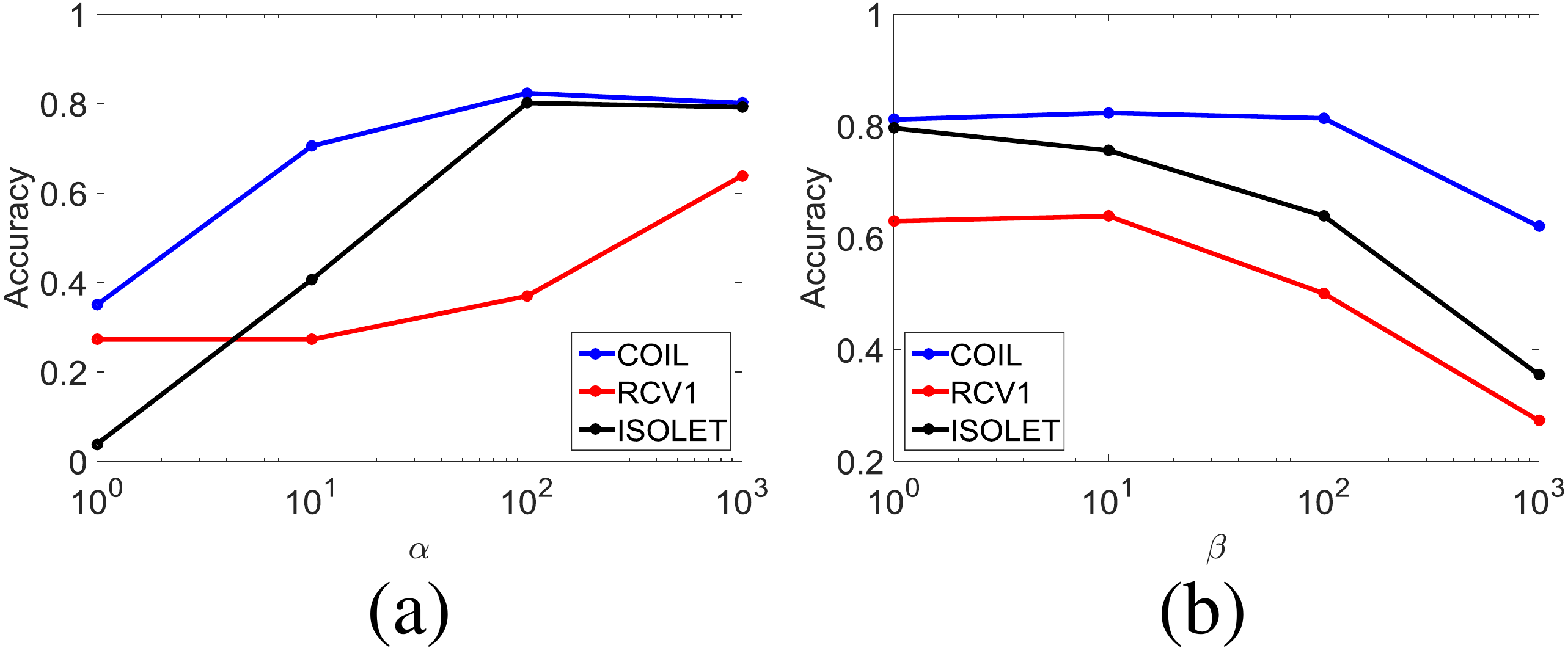}\\
  \caption{Parametric sensitivity of the proposed SIIS. (a) and (b) plot the accuracy of SIIS w.r.t. the variation of $\alpha$ and $\beta$, respectively.}\label{fig:ParametricSensitivity}
  \vskip -10pt
\end{figure}

\section{Conclusion}
To solve the label shortage and label inaccuracy that often occur in many real-world problems, this paper proposed a novel graph-based SSL algorithm dubbed ``Semi-supervised learning under Inadequate and Incorrect Supervision'' (SIIS). Two measures, namely graph trend filtering and smooth eigenbase pursuit, are formulated into a unified optimization framework to tackle the label errors. This optimization model was solved via the Alternating Direction Method of Multipliers, of which the convergence has been theoretically proved. We tested our SIIS on image, text and audio datasets under different levels of label noise, and found that SIIS performs robustly to label noise and achieves superior performance to other compared baseline methods. In particular, SIIS is able to obtain very encouraging results when more than half of the limited labeled examples are mislabeled, which further demonstrates the effectiveness and robustness of the proposed algorithm.





%

%
%

{
\bibliographystyle{IEEEtran}
\bibliography{TLLTRef}
}

\end{document}